\documentclass{article}

\PassOptionsToPackage{dvipsnames}{xcolor}

\usepackage{iclr2025_conference,times}

\newcommand*\diff{\mathop{}\!\mathrm{d}}

\usepackage{microtype}
\usepackage{graphicx}
\usepackage{subfigure}
\usepackage{booktabs}       
\usepackage{amsfonts}       
\usepackage{nicefrac}       
\usepackage{enumitem}
\usepackage{wrapfig}
\usepackage{multirow}

\usepackage{hyperref}

\usepackage{amsmath}
\usepackage{amssymb}
\usepackage{mathtools}
\usepackage{amsthm}
\usepackage{algorithm2e}
\RestyleAlgo{ruled}

\usepackage{dsfont}

\usepackage[capitalize,noabbrev]{cleveref}

\newtheorem{theorem}{Theorem}

\newtheorem{lemma}{Lemma}

\newtheorem{assumption}{Assumption}
\newcommand{\E}{\mathbb{E}}
\newcommand{\R}{\mathbb{R}}

\newcommand{\prob}{\mathbb{P}}
\newcommand{\indep}{\perp \!\!\! \perp}

\newcommand{\mathup}[1]{\text{\textup{#1}}}
\usepackage{xspace}

\usepackage{bbm}

\usepackage[utf8]{inputenc} 
\usepackage[T1]{fontenc}    
\usepackage{hyperref}       
\usepackage{url}            
\usepackage{booktabs}       
\usepackage{amsfonts}       
\usepackage{nicefrac}       
\usepackage{microtype}      

\usepackage[para]{footmisc}
\usepackage{ragged2e}

\usepackage{tikz}
\usepackage{graphicx}

\newcommand{\bound}[1]{{\color{blue}#1}}

\newcommand*\circledgreen[1]{%
\tikz[baseline=(char.base)]{
  \node[shape=circle, draw=ForestGreen!60, fill=ForestGreen!10, thick, inner sep=1pt] (char) {\scriptsize\textsf{#1}};
}}

\newcommand*\circled[1]{%
\tikz[baseline=(char.base)]{
  \node[shape=circle, draw=NavyBlue!60, fill=NavyBlue!10, thick, inner sep=1pt] (char) {\scriptsize\textsf{#1}};
}}

\newcommand*\circledgray[1]{%
\tikz[baseline=(char.base)]{
  \node[shape=circle, draw=black!60, fill=black!10, thick, inner sep=1pt] (char) {\scriptsize\textsf{#1}};
}}

\title{Learning Representations of Instruments for Partial Identification of Treatment Effects}

\centering
\author{Jonas Schweisthal\textsuperscript{1, 2, 5}  \And Dennis Frauen\textsuperscript{1, 2} \And Maresa Schröder\textsuperscript{1, 2} \And Konstantin Hess\textsuperscript{1, 2} \And Niki Kilbertus\textsuperscript{2, 3, 4} 
\And Stefan Feuerriegel\textsuperscript{1, 2} }

\iclrfinalcopy 
\begin{document}

\maketitle
\addtocounter{footnote}{1}
\footnotetext{LMU Munich}
\addtocounter{footnote}{1}
\footnotetext{Munich Center for Machine Learning}
\addtocounter{footnote}{1}
\footnotetext{School of Computation, Information and Technology, TU Munich}
\addtocounter{footnote}{1}
\footnotetext{Helmholtz Munich}
\addtocounter{footnote}{1}
\footnotetext{Corresponding author (\texttt{jonas.schweisthal@lmu.de})}
\justifying

\begin{abstract}
Reliable estimation of treatment effects from observational data is important in many disciplines such as medicine. However, estimation is challenging when unconfoundedness as a standard assumption in the causal inference literature is violated. In this work, we leverage arbitrary (potentially high-dimensional) instruments to estimate bounds on the conditional average treatment effect (CATE). Our contributions are three-fold: (1)~We propose a novel approach for partial identification through a mapping of instruments to a discrete representation space so that we yield valid bounds on the CATE. This is crucial for reliable decision-making in real-world applications. (2)~We derive a two-step procedure that learns tight bounds using a tailored neural partitioning of the latent instrument space. As a result, we avoid instability issues due to numerical approximations or adversarial training. Furthermore, our procedure aims to reduce the estimation variance in finite-sample settings to yield more reliable estimates. (3)~We show theoretically that our procedure obtains valid bounds while reducing estimation variance. We further perform extensive experiments to demonstrate the effectiveness across various settings. Overall, our procedure offers a novel path for practitioners to make use of potentially high-dimensional instruments (e.g., as in Mendelian randomization). 
\end{abstract}

\section{Introduction}
Estimating the {conditional average treatment effect (CATE)} from observational is an important task for personalized decision-making in medicine \citep{feuerriegel2024causal}. For example, a common question in medicine is to estimate the effect of alcohol consumption on the onset of cardiovascular diseases \citep{holmes2014association}. There are many reasons, including costs and ethical concerns, why CATE estimation is often based on observational data (such as, e.g., electronic health records, clinical registries).  

\begin{wrapfigure}[14]{r}{0.36\textwidth}
    \centering
    \vspace{-.5cm}
    \includegraphics[width=\linewidth]{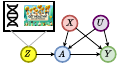}
    \caption{Overview of the IV setting. We consider complex instruments $Z$ (e.g., gene data, text, images), observed confounders $X$, unobserved confounders $U$, a binary treatment $A$, and an outcome $Y$.
    \label{fig:setting}
    }
    \vspace{-3cm}
\end{wrapfigure}

However, identifying the CATE from observational data is challenging as it typically requires \emph{strong} assumptions in the form of \emph{unconfoundedness} \citep{Rubin.1974}.  Unconfoundedness assumes there exist no additional unobserved confounders $U$ between treatment $A$ and outcome $Y$. If the unconfoundedness assumption is violated, a common strategy is to leverage \textbf{instrumental variables~(IVs)} $Z$. IVs affect only the treatment $A$ but exclude unobserved confounding between $Z$ and $Y$, which often can be ensured by design such as for randomized studies with non-compliance \citep{Imbens.1994}. The causal graph for the IV setting is shown in Fig.~\ref{fig:setting}. 

\textbf{Motivational example:} \emph{Mendelian randomization.}  Mendelian randomization \citep{pierce2018mendelian} refers to the use of genetic information as instruments $Z$ to estimate the effect of a treatment or exposure $A$ (e.g., alcohol consumption) on some medical outcome $Y$ (e.g., cardiovascular diseases). In this setting, there are further patient characteristics that are observed ($X$) but also unobserved ($U$), which one accounts for through the instrument. Yet, common challenges are that \textbf{(i)} instruments with genetic information are often \emph{high-dimensional} and \textbf{(ii)} involve \emph{complex, non-linear relationships between instruments and treatment intake or exposure}. \hfill

However, existing IV methods using machine learning for point estimation of the CATE rely on \emph{strong simplifying assumptions} ($\rightarrow$ violating \textbf{(ii)} from above). For example, some methods assume linearity in some feature space in the CATE and make other, strict parametric assumptions on the unobserved confounders such as additivity or homogeneity \citep{Hartford.2017, Singh.2019, Xu.2021}. Yet, such simplifying assumptions are often \emph{not} realistic and can even lead to unreliable and false conclusions by the mis-specification of the CATE. 

A potential remedy is to use IVs for \textbf{partial identification} of the CATE where one circumvents any hard parametric assumptions by estimating upper and lower bounds of the CATE \citep{Manski.1990}. This is usually sufficient in medical practice when one is merely interested in whether a treatment variable (e.g., exposure as in Mendelian randomization) has a positive or a negative effect. So far, methods for partial identification of the CATE in IV settings are rare. There exist closed-form bounds (i.e., via a fixed target estimand that can be learned), yet only for the setting with \underline{both} \emph{discrete} instruments and \emph{discrete} treatments \citep{Balke.1997}.\footnote{Originally, the paper \citep{Balke.1997} derives the bounds for the average treatment effect (ATE) but which can be straightforwardly extended to the CATE.} 

Existing machine learning methods for partial identification are typically designed for \emph{simple} instruments that are binary or discrete ($\rightarrow$ violating \textbf{(i)} from above).\footnote{In Mendelian randomization, this then requires further simplifications. For example, many applications of Mendelian randomization in medicine treat the genetic variants as \emph{discrete} instruments (sometimes also by grouping) or use a predefined risk or polygenic score as a \emph{one-dimensional} continuous instrument \citep{pierce2018mendelian}. Hence, rich information from complex gene data is lost.} Alternatively, methods that extend partial identification for continuous instruments require \emph{unstable} training paradigms such as adversarial learning \citep{Kilbertus.2020, Padh.2023} which becomes even more unstable for more complex instruments.  In contrast, there is a scarcity of methods that can deal robustly with continuous, as well as \emph{complex} and potentially high-dimensional instruments such as, e.g., gene expressions as in Mendelian randomization but also text, images, or graphs.\footnote{In Appendix~\ref{app:RW_assumptions}, we provide an extended discussion about the real-world relevance and applicability of our method.}

\textbf{Our paper:} In this work, we leverage complex instruments for partial identification of the CATE. Specifically, we allow for instruments that can be continuous and potentially high-dimensional (such as gene information) and, on top of that, we explicitly allow for complex, non-linear relationships between instruments and treatment intake or exposure. In the rest of this paper, we refer to this setting as ``complex'' instruments. 

To this end, we proceed as follows. (1)~We propose a novel approach for partial identification through a mapping of complex instruments to a discrete representation space so that we yield valid bounds on the CATE. We motivate our approach in Fig.~\ref{fig:motivation}. (2)~We derive a two-step procedure that learns tight bounds using a neural partitioning of the latent instrument space. As a result, we avoid instability issues due to numerical approximations or adversarial training, which is a key limitation of prior works. We further improve the performance of our procedure by explicitly reducing the estimation variance in finite-sample settings to yield more reliable estimates. (3)~We provide a theoretical analysis of our procedure and perform extensive experiments to demonstrate the effectiveness across various settings.

\textbf{Contributions:}\footnote{Code is available at \url{https://github.com/JSchweisthal/ComplexPartialIdentif}.} (1)~To the best of our knowledge, this is the first IV method for partial identification of the CATE based on complex instruments. (2)~We derive a two-step procedure to learn tight bounds. (3)~We demonstrate the effectiveness of our method both theoretically and numerically.

\section{Related Work}\label{sec:related_work}

\textbf{Machine learning for CATE estimation with IV:} Existing works have different objectives. One literature stream leverages IVs for CATE estimation but focuses on settings where the treatment effect can be identified from the data. This includes work that extends the classical two-stage least-squares estimation to non-linear settings by learning non-linear feature spaces \citep{Singh.2019, Xu.2021}, deep conditional density estimation in the first stage \citep{Hartford.2017}, or using moment conditions \citep{Bennett.2019}. Another literature stream aims at new machine learning methods with favorable properties such as being doubly robust \citep{Kennedy.2019, Ogburn.2015, Semenova.2021, Syrgkanis.2019} or multiply robust \citep{Frauen.2023b}.  

\begin{wrapfigure}[20]{r}{0.7\textwidth}
    \centering
    \includegraphics[width=\linewidth]{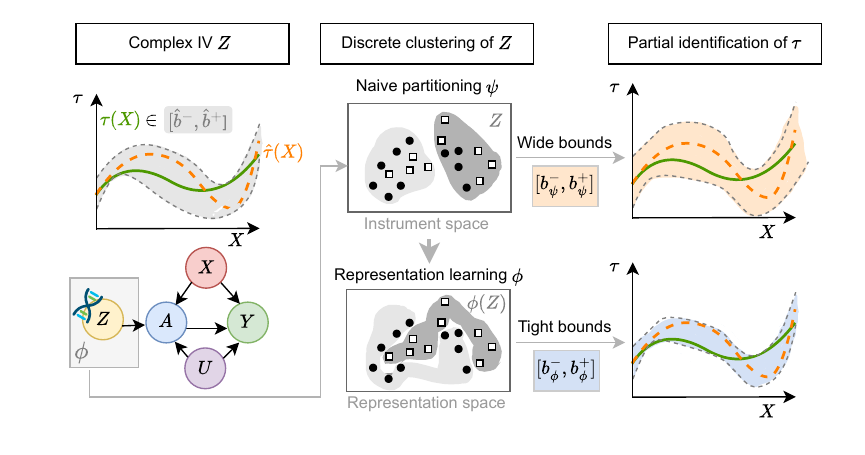}
    \vspace{-.9cm}
    \caption{Leveraging complex instruments for partial identification of the CATE through discrete representations of $Z$. Naive discretization on the IV input space leads to wide, and thus non-informative, bounds. Our method learns a latent representation $\phi(Z)$ to yield tight bounds.}
    \label{fig:motivation}
\end{wrapfigure}

Recently, researchers started applying machine learning methods to IVs from Mendelian randomization \citep{legault2024novel, malina2022deep}, which is our motivational example from above. However, these works aim at point-identified CATE estimation with IVs. As a result, these rely on \emph{hard} and generally untestable assumptions on some effects in the causal graph, such as linearity, monotonicity, additivity, or homogeneity \citep{ Wang.2018}. This is unlike our method for partial identification that does \emph{not} require such hard assumptions and that is non-parametric.

\textbf{Partial identification:} Partial identification aims to identify and learn upper and lower bounds of some causal quantity (e.g., the CATE) when the causal quantity itself cannot be point identified from the data and assumptions. In a general setting with binary treatments, \citet{robins1989analysis} and \citet{Manski.1990} derived closed-form bounds on the ATE for bounded outcomes $Y$. Further work extended these ideas to settings with binary instrumental variables, binary treatments, and binary outcomes \citep{Balke.1994, Balke.1997} to derive tighter bounds. Newer approaches for discrete variables include the works of \citet{Duarte.2023} and \citet{Guo.2022}. \citet{swanson2018partial} provide an extensive overview of partial identification in this setting. Other works focus on how to leverage additional observed confounders to further tighten bounds on the ATE  \citep[see, e.g.,][]{levis2023covariate}. However, these works do \underline{not} focus on efficiently leveraging continuous or even high-dimensional instruments for learning tight bounds, unlike our work that is tailored to such complex instruments.

Another literature stream focuses on partial identification under general causal graphs \citep{Balazadeh.2022}, including IV settings with continuous variables such as continuous treatments \citep{Gunsilius.2020, Hu.2021, Kilbertus.2020, Padh.2023}. However, these methods either make strong assumptions about the treatment response functions or require unstable optimization via adversarial training and/or generative modeling such as through using GANs. This can easily result in \emph{unreliable} estimates of bounds for finite data, especially with high-dimensional instruments. Further, these methods are \emph{not} directly tailored for binary treatments, unlike our method.

\textbf{Research gap:} To the best of our knowledge, reliable machine learning methods for partial identification of the CATE with complex instruments are missing. To draw conclusions about CATEs (as in, e.g., Mendelian randomization), our method  is the first to: (i)~make use of the complex instrument information (e.g., continuous or high-dimensional), (ii)~avoid making strong parametric assumptions by focusing on partial identification, and (iii)~avoid unstable training procedures such as adversarial learning.

\section{Problem Setup}
\textbf{Setting:} We focus on the standard IV setting \citep{Angrist.1996,Wooldridge.2013}. Hence, we consider instruments (e.g., gene data, text, images) given by $Z \in  \mathcal{Z} \subseteq \R^d$ but, unlike previous research, allow the instruments to be complex. As such, we allow the instruments to be continuous and potentially high-dimensional. We further have access to an observational dataset $\mathcal{D} = \left\{z_i, x_i, a_i, y_i\right\}_{i=1}^{n}$ of size $n$. The data is sampled i.i.d. from a population $(Z, X, A, Y) \sim \mathbb{P}$, with observed confounders $X \in \mathcal{X} \subseteq \R^p$, binary treatments $A \in \mathcal{A} \subseteq \{0, 1\}$, and bounded outcomes $Y \in \mathcal{Y} \subseteq [s_1, s_2] \subseteq \R$. Additionally, we allow for unobserved confounders $U$ of arbitrary form between $A$ and $Y$. 

We further assume a causal structure as shown in Fig.~\ref{fig:setting}. In particular, we assume that $Z$ is an instrumental variable that has an effect on the treatment $A$ but no direct effect on the outcome $Y$ except through $A$. Further, we assume that $Z$ is independent of $X$, e.g., by randomization. In Appendix~\ref{app:RW_assumptions}, we provide an extended discussion to show the real-world relevance and validity of our assumptions in different settings.

\textbf{Notation:} Throughout our work, we denote the \emph{response function} by $\mu^a(x, z) := \mathbb{E}[Y|X=x, A=a, Z=z]$ and the \emph{propensity score} by $\pi(x, z) := \mathbb{P}(A=1 | X=x, Z=z)$.


\textbf{CATE:} We use the potential outcomes framework \citep{Rubin.1974} to formalize our causal inference problem. Let $Y(a) \in \mathcal{Y}$ denote the potential outcome under treatment $A=a$. We are thus interested in the CATE $\tau(x) = \mathbb{E}[Y(1) - Y(0) | X=x]$.

\textbf{Identifiability:} We make the following standard assumptions from the literature in partial identification with IVs \citep{Angrist.1996}. 
\begin{assumption}[Consistency]\label{ass:consistency}  $Y(A) = Y$.
 \end{assumption}
 \begin{assumption}[Exclusion]\label{ass:exclsion}
 $Z\indep Y(A) \mid (X, A, U)$, where $U$ is an unobserved confounder between $A$ and $Y$. This implies that $Z$ affects $Y$ only through $A$ and there is no additional unobserved confounding between $Z$ and $Y$.
 \end{assumption}
 \begin{assumption}[Independence]\label{ass:independence}
     $Z \indep (U, X)$.
 \end{assumption}

Note that, however, Assumptions~\ref{ass:consistency}--\ref{ass:independence} from the standard IV setting are \emph{not} sufficient to ensure identifiability of the CATE \citep{Gunsilius.2020}. To ensure identifiability, one would require additional assumptions, such as linearity or, more generally, additive noise assumptions \citep{Hartford.2017, Wang.2018}. Yet, such assumptions are highly restrictive and are neither testable nor typically ensured in real-world scenarios. Hence, this motivates our objective to perform partial identification instead. 

\textbf{Objective}: We frame our objective as a \emph{partial identification} problem and thus focus on estimating \emph{valid} bounds $( b^-(x) , b^+(x))$ for the CATE $\tau(x)$ such that $b^-(x) \leq \tau(x) \leq b^+(x)$ holds for all possible $x \in \mathcal{X}$. Furthermore, the bounds should be \emph{informative}, i.e., we would like to minimize the expected bound width $\mathbb{E}_X[b^+(X) -  b^-(X)]$, while still ensuring validity. Formally, we aim to solve
\begin{equation}\label{eq:objective_population}
    b_\ast^-, b_\ast^+  \in  \; \underset{b^-, b^+}{\arg\min} \, \mathbb{E}_X [b^+(X) - b^-(X)] \quad \text{s.t.} \quad b^-(x) \leq \tau(x) \leq b^+(x) \quad \text{for all} \quad x \in \mathcal{X}.
\end{equation} 

\section{Partial identification of CATE with complex instruments}

\subsection{Overview}

We now present our proposed method to solve the partial identification problem from Eq.~\eqref{eq:objective_population}. Solving Eq.~\eqref{eq:objective_population} directly is \emph{infeasible} because it involves the unknown CATE $\tau(x)$. Hence, we propose the following approach:

\begin{tabular}{|p{.99\textwidth}}
\textbf{Outline:} \circledgreen{1}~We learn a discretized representation (also called partitioning) $\phi(Z)$ of the instrumental variable $Z$. \circledgreen{2}~We then derive closed-form bounds given the discrete representation $\phi$.
\circledgreen{3}~We transform the closed-form bounds back to our original bounding problem and, in particular, express all quantities involved as quantities that can be estimated from observational data.
\end{tabular}

Below, we first explain why existing closed-form bounds are \emph{not} directly applicable and why deriving such bounds is non-trivial. We then proceed by providing the corresponding theory for the above method. Specifically, we first take a population view to show theoretically that our bounds are valid (Sec.~\ref{sec:population_view}). Then, we take a finite-sample view and present an estimator (Sec.~\ref{sec:finite-sample_view}). 

\emph{Limitations of existing bounds:} There exist different approaches for bounding treatment effects (see Sec.~\ref{sec:related_work}) using continuous instruments, yet these either require additional assumptions or can easily become unstable, especially for high-dimensional $Z$. Furthermore, these bounds consider continuous treatments but are \emph{not tailored} for binary treatments (e.g., whether a drug is administered). Hence, we derive custom bounds for our setting.

\emph{Why is the derivation non-trivial?} For binary treatments, it turns out that there exist closed-form solutions for bounds whenever the instrument $Z$ is discrete. That is, the existing bounds for the average treatment effect (ATE) with continuous bounded outcome proposed in \citep{Manski.1990} can be extended to non-parametric closed-form bounds for the CATE \citep{schweisthal2024meta-learners}. While these bounds are useful in a setting with discrete instruments $Z$, they are \underline{not} directly applicable to continuous or even high-dimensional $Z$ due to two main reasons: (1)~The bounds need to be evaluated for \emph{all} combinations $l, m \in \mathcal{Z}^2 \subseteq \R^d \times \R^d$, which is \emph{intractable}. (2)~Evaluating the bounds only on a random subset of combinations $l , m$ can result in \emph{arbitrary high} estimation variance for regions with a low joint density of $p(X=x, Z=l)$ or $p(X=x, Z=m)$.
Hence, we must derive a novel method for estimating bounds based on complex instruments (that are, e.g., continuous or high-dimensional), yet this is a highly \emph{non-trivial} task. 

\subsection{Population view}
\label{sec:population_view}

In the following theorem, we provide a novel theoretical result of how to obtain valid bounds based on discrete representations $\Phi(Z)$ of the instrument $Z$.

\begin{theorem}[Bounds for arbitrary instrument discretizations]\label{thrm:bounds_phi} Let
$\phi: \mathcal{Z} \xrightarrow{} \{0, 1, \ldots, k\}$ be an arbitrary mapping from the high-dimensional instrument $Z$ to a discrete representation. We define
\begin{equation} 
\mu_{\phi}^a(x, \ell) =  \int_Z \frac{\mu^a(x, z) \mathbb{P}(\phi(Z)=\ell| Z=z)}{\mathbb{P}(A=a, \phi(Z)=\ell)}  \mathbb{P}(A=a| Z=z) \mathbb{P}(Z=z) \diff z   \quad \text{and}
\end{equation}
\begin{equation} 
\pi_{\phi}(x, \ell) =  \int_Z \frac{\pi(x, z) \mathbb{P}(\phi(Z)=\ell | Z=z)}{\mathbb{P}(\phi(Z)=\ell)} \mathbb{P}(Z=z)\diff z.
\end{equation}
Then,  under Assumptions~\ref{ass:consistency}, \ref{ass:exclsion}, and \ref{ass:independence}, the CATE $\tau(x)$ is bounded by
\begin{equation}
b_\phi^-(x) \leq \tau(x) \leq b_\phi^+(x),
\end{equation}
with
\begin{equation}\label{eq:bounds_phi}
    b_\phi^+(x) = \min_{l, m} b^+_{\phi; l, m}(x) \quad \text{and} \quad b_\phi^-(x) = \max_{l, m} b^-_{\phi;l, m}(x) 
\end{equation}
where
\begin{equation}
    b^+_{\phi;l, m}(x) =\pi_\phi(x, l) \mu_\phi^{1}(x, l) + (1 - \pi_\phi(x, l)) s_2   - (1-{\pi}_\phi(x, m)) {\mu}_\phi^{0}(x, m) -  {\pi}_\phi(x, m) s_1 ,
\end{equation}
\vspace{-0.4cm}
\begin{equation}
    b^-_{\phi;l, m}(x) ={\pi}_\phi(x, l) {\mu}_\phi^{1}(x, l) + (1 - {\pi}_\phi(x, l)) s_1   - (1-\pi_\phi(x, m)) {\mu}_\phi^{0}(x, m) -  {\pi}_\phi(x, m) s_2.
\end{equation}
\end{theorem}
\begin{proof}
See Appendix~\ref{app:proofs}. 
\end{proof}

Theorem~\ref{thrm:bounds_phi} states that, in population, we yield valid closed-form bounds for $\tau(x)$ for arbitrary representations $\phi$. In particular, we can relax the optimization problem from Eq.~\eqref{eq:objective_population} and obtain valid bounds $b_\bound{\phi^\ast}^+(X) \geq b_\ast^+(X)$ and $b_\bound{\phi^\ast}^-(X) \leq b_\ast^-(X)$ by solving
\begin{equation}\label{eq:objective_population_repr}
    \phi^\ast \in \underset{\phi \in \Phi}{\arg\min} \; \mathbb{E}_X[b_\bound{\phi}^+(X) - b_\bound{\phi}^-(X)].
\end{equation}
Here, we highlight the dependence of variables on the representation $\bound{\phi}$ in \bound{blue} to show the differences to Eq.~\eqref{eq:objective_population}. Note the following differences: In contrast to Eq.~\eqref{eq:objective_population}, we do not impose any validity constraints in Eq.~\eqref{eq:objective_population_repr} because Theorem~\ref{thrm:bounds_phi} automatically ensures the validity of our bounds. Furthermore, in contrast to Eq.~\eqref{eq:objective_population}, the objective from Eq.~\eqref{eq:objective_population_repr} only depends on identifiable quantities that can be estimated from observational data. 

\textbf{Implications of Theorem \ref{thrm:bounds_phi}:} Our derivation of closed-forms bounds for arbitrary discrete representations of complex $Z$ comes with an important additional benefit: The bounds only depend on (i) discrete probabilities, (ii) quantities which are independent of $\phi$ and thus do not change for different $\phi$, and (iii) the discrete representation mapping to be learned itself. As a result, this allows us to \emph{directly} learn $\phi$ wrt. Eq.~\eqref{eq:objective_population_repr}. As such, we circumvent the need for adversarial or alternating training, which results in more robust estimation. 

\subsection{Finite-sample view}
\label{sec:finite-sample_view}

In practice, we have to estimate the bounds from Theorem~\ref{thrm:bounds_phi} from finite observational data. For this purpose, we start with arbitrary initial estimators: $\hat{\pi}(x, z)$ is the estimator of the propensity score ${\pi}(x, z)$, $\hat{\mu}^a(x, z)$ of the response function ${\mu}^a(x, z)$, and $\hat{\eta}(z)$ of ${\eta}(z) = \mathbb{P}(A = 1 \mid Z = z)$. 

Once the initial estimators are obtained, we can estimate our second-stage nuisance functions defined in Eq.~\eqref{eq:mu_transform} and \eqref{eq:pi_transform} via
\begin{equation}\label{eq:estimates_nuisance_mu}
    \hat{\mu}_{\phi}^{a}(x, \ell) = \frac{1}{\sum_{j=1}^{n} \mathds{1}\{{\phi}(z_j)=\ell, a_j=a\}} \sum_{j=1}^{n}\hat{\mu}^{a}(x, z_j)   \mathds{1}\{{\phi}(z_j)=\ell\} (a\hat{\eta}(z_j)+(1-a)(1-\hat{\eta}(z_j)))
\end{equation}
and via
\begin{equation}\label{eq:estimates_nuisance_phi}
    \hat{\pi}_{\phi}(x, \ell)  = \frac{1} {\sum_{j=1}^{n} \mathds{1}\{\phi(z_j)=\ell\}}  \sum_j^{n}  \hat{\pi}(x, z_j)  \mathds{1}\{\phi(z_j)=\ell\}.
\end{equation}
Finally, we can directly `plug in' these estimators into Eq.~\eqref{eq:bounds_phi} to compute estimates of the upper and lower bound $\hat{b}_{\phi}^-(x), \hat{b}_{\phi}^+(x)$.

A na{\"i}ve approach would now directly use $(\hat{b}_{\phi}^-(x), \hat{b}_{\phi}^+(x))$ to solve the optimization in Eq.~\eqref{eq:objective_population_repr}. However, for finite samples, it turns out this is infeasible without restricting the complexity of the representation function. The reason is outlined in the following theoretical results.

\begin{lemma}[Tightness-bias-variance tradeoff]\label{lem:error_decomposition}
Let $\mathbb{E}_n$ and $\mathup{Var}_n$  denote the expectation and variance with respect to the observational data (of size $n$). Then, it holds
\begin{equation}
    \mathbb{E}_n \Big[\left(b_\ast^+(x) - \hat{b}_{\phi^\ast}^+(x)\right)^2\Big] \leq 2 \Big( \underbrace{\left(b_\ast^+(x) - {b}_{\phi^\ast}^+(x)\right)^2}_{\text{(i) Population tightness}} + \underbrace{\mathbb{E}_n\left[b_{\phi^\ast}^+(x) - \hat{b}_{\phi^\ast}^+(x)\right]^2}_{\text{(ii) Estimation bias}} + \underbrace{\mathup{Var}_n(\hat{b}_{\phi^\ast}^+(x))}_{\text{(iii) Estimation variance}} \Big).
\end{equation}
\end{lemma}
\begin{proof}
    See Appendix~\ref{app:proofs}.
\end{proof}

\textbf{Interpretation of Lemma~\ref{lem:error_decomposition}:} Lemma~\ref{lem:error_decomposition} shows that the mean squared error (MSE) between the estimated representation-based bound $\hat{b}_{\phi^\ast}^+(x)$ and the ground-truth optimal bound $b_\ast^+(x)$ can be decomposed into the following three components: (i)~\emph{population tightness}, (ii)~\emph{estimation bias}, and (iii)~\emph{estimation variance}. $\bullet$\,Term~(i) describes the discrepancy between the tightest achievable representation-based bound ${b}_{\phi^\ast}^+$ and the ground-truth bound $b_\ast^+(x)$. It will \emph{decrease} if we allow for more complex representations $\Phi$, for example by increasing the number of partitions $k$. $\bullet$\,Term~(ii) describes the estimation bias due to using finite-sample estimators for estimating the bounds. It will generally depend on the type of estimators we employ for $\hat{\pi}(x, z)$, $\hat{\mu}^a(x, z)$, and $\hat{\eta}(z)$. $\bullet$\,Finally, term (iii) characterizes the variance due to using finite-sample estimators. In contrast to term~(i), it will \emph{increase} when we allow the representation to be more complex.

To make point (iii) more explicit, we derive the asymptotic distributions of the estimators from Eq.~\eqref{eq:estimates_nuisance_mu} and Eq.~\eqref{eq:estimates_nuisance_phi} that are used to estimate the final bounds.

\begin{theorem}[Asymptotic distributions of estimators]\label{thrm:variances}
It holds that
\begin{align}
\sqrt{n} \hat{\mu}_{\phi}^{a}(x, \ell) & \xrightarrow{d} \mathcal{N}\left(\mu_{\phi}^{a}(x, \ell), \frac{1}{p_{\ell, \Phi}}\left(\frac{\operatorname{Var}(g(Z) \mid \Phi(Z) = \ell)}{c} + d \right) \right)  \\  
\sqrt{n} \hat{\pi}_{\phi}(x, \ell) & \xrightarrow{d} \mathcal{N}\left({\pi}_{\phi}(x, \ell), \frac{1}{p_{\ell, \Phi}}\operatorname{Var}(h(Z) \mid \Phi(Z) = \ell) \right)
\end{align}    
for $c, d > 0$ and where $p_{\ell, \Phi} = \mathbb{P}(\Phi(Z) = \ell)$,  $g(Z) = \hat{\mu}^{a}(x, Z) (a \hat{\eta}(Z) + (1 - a)(1 - \hat{\eta}(Z))$, and $h(Z) = \hat{\pi}(x, Z)$.
\end{theorem}
\begin{proof}
    See Appendix~\ref{app:proofs}.
\end{proof}

We observe that the variance of the estimators (and, thus, of the estimated bounds) explodes for small values of  $p_{\ell, \Phi} = \mathbb{P}(\Phi(Z) = \ell)$. Hence, to reduce the estimation variance, we aim to learn a representation $\phi$ that avoids low $p_{\ell, \Phi}$ for some $\ell$, e.g., by limiting the number of partitions $k$. 
$\mathbf{\Rightarrow}$ Altogether, as a consequence of Lemma~\ref{lem:error_decomposition} and Theorem~\ref{thrm:variances}, we obtain an \emph{inherent trade-off between tightness of the bounds in populations and estimation variance in finite-samples}.

\textbf{Learning objective for the representation $\phi$:} Due to the inherent trade-off between tightness of the bounds and estimation variance, the aim for learning the representation $\phi$ is two-fold. On the one hand, we \textbf{(a)} aim to learn tight bounds, which is given in the objective in Eq.~\eqref{eq:objective_population_repr}. On the other hand, we \textbf{(b)}~also have to account for controlling the variance in finite-sample settings, especially for high-dimensional $Z$. Motivated by Theorem~\ref{thrm:variances}, we ensure $\hat{p}_{\ell, \Phi} > \varepsilon$ for some $\varepsilon > 0$, where $\hat{p}_{\ell, \Phi}$ is an estimator of $p_{\ell, \Phi} = \mathbb{P}(\Phi(Z) = \ell)$. Combining both \textbf{(a)} and \textbf{(b)} yields the following objective:
\begin{equation}\label{eq:objective_finite_sample}
    \phi^* \in \underset{\phi \in \Phi}{\arg\min} \; \mathbb{E}_X[\hat{b}_\phi^+(X) - \hat{b}_\phi^-(X)] \quad \text{s.t.} \quad \hat{p}_{\ell, \Phi} > \varepsilon,
\end{equation} 
for some $\varepsilon > 0$ and all $\ell \in \{1,\dots,k\}$. 
We next present a neural method to learn tight bounds using the above objective.

\section{Neural method for learning CATE bounds with complex instruments}
\label{sec:neural_method}

In this section, we propose a neural method for our objective to learn tight and valid bounds. Our method consists of two separate stages (see Algorithm~\ref{alg:learning_full}): \circled{1}~we learn initial estimators of the three nuisance functions, and \circled{2}~we learn an optimal representation $\phi^\ast$, so that the width of the bounds is minimized.  Note that our method is completely model-agnostic. Hence, arbitrary machine learning models can be used in the first and second stages in order to account for the properties of the data. For example, for instruments with gene data, one could use pre-trained encoders to further optimize the downstream performance.  We give an overview of the workflow of our method in Fig.~\ref{fig:architecture} (see Algorithm~\ref{alg:main_algorithm} for pseudocode).

\begin{figure}[h]
\vspace{-0.3cm}
    \centering
    \includegraphics[width=.90\textwidth]{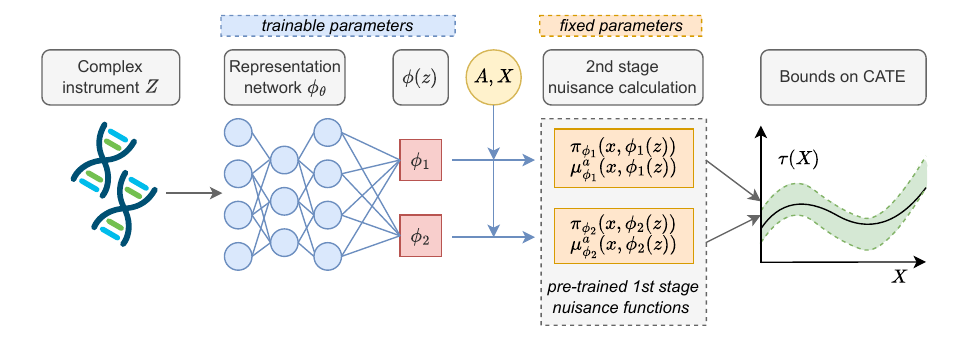}
    \vspace{-0.3cm}
    \caption{Workflow of the second stage of our method for calculating bounds on the CATE: The representation network $\phi_\theta$ learns discrete latent representations of the complex $Z$ (e.g., continuous or high-dimensional). By employing the pre-trained $\hat{\mu}$, $\hat{\pi}$, and $\hat{\eta}$, we can directly calculate the nuisance estimates conditional on the latent representation $\phi(z)$ by using Eq.~\eqref{eq:estimates_nuisance_mu} and Eq.~\eqref{eq:estimates_nuisance_phi} to yield the bounds.}
    \vspace{-0.2cm}
    \label{fig:architecture}
\end{figure}

\circled{1}~\textbf{Initial nuisance estimation:}  In the first stage, we can use arbitrary machine learning models (e.g., feed-forward neural network) to learn the first-stage nuisance functions $\hat{\mu}^{a}(x, z) = \hat{\E}[Y \mid X = x, A = a, Z=z]$ , 
$\hat{\pi}(x, z) = \hat{\prob}(A = 1 \mid X = x, Z=z)$, and
$\hat{\eta}(z) = \hat{\prob}(A=1 \mid Z=z)$.

Recall that we consider $Z$ and $X$, which are both potentially high-dimensional. Hence, for  $\hat{\mu}^{a}(x, z)$ and $\hat{\pi}(x, z)$, we use network architectures that have (i) different encoding layers for $X$ and $Z$, so that we capture structured information within the variables and (ii) shared layers on top of the encoding to learn common structures. Further, for $\hat{\mu}^{a}(x, z)$, we use two outcome heads for both treatment options $A \in \{0, 1\}$ to ensure that the influence of the treatment on the outcome prediction does not `get lost' in the high-dimensional space of $X$ and $Z$ \citep{Shalit.2017}. 

\begin{algorithm}\label{alg:main_algorithm}
\DontPrintSemicolon
\scriptsize
\caption{Two-stage learner for estimating bounds with complex instruments}
\label{alg:learning_full}
\SetKwInOut{Input}{Input}
\SetKwInOut{Output}{Output}
\Input{{observational data sampled from $(Z, X, A, Y)$, epochs $E$, batch size $N_b$, 
neural network ${\phi}_\theta$ with parameters $\theta$}, learning rate $\delta$}
\Output{bounds $\hat{b}_{\phi_\theta}^-(x), \hat{b}_{\phi_\theta}^+(x)$}
\tcp{First stage (nuisance estimation)}
$\hat{\mu}^{a}(x, z) \gets \hat{\E}[Y \mid X = x, A = a, Z=z]$\;
$\hat{\pi}(x, z) \gets \hat{\prob}(A = 1 \mid X = x, Z=z)$\;
$\hat{\eta}(z) \gets \hat{\prob}(A=1 \mid Z=z)$\;
\tcp{Second-stage (partition learning and bound calculation)}

     \For{$e \in \{1, ..., E\}$ in batches}{
     \For{$\ell \in \{1, ..., k\}$}{

   $\hat{\mu}_{\phi_\theta}^{a}(x, \ell) = \frac{N_b}{\sum_j^{N_b} \mathds{1}\{{\phi}_\theta(z_j)=\ell, A=a)\}} \sum_j^{N_b}\hat{\mu}^{a}(x, z_j)   \mathds{1}\{{\phi}_\theta(z_j)=\ell\} (a\hat{\eta}(z_j)+(1-a)(1-\hat{\eta}(z_j)))$ \; 
 $\hat{\pi}_{\phi_\theta}(x, \ell)  = \frac{N_b} {\sum_j^{N_b} \mathds{1}\{\phi_\theta(z_j)=\ell\}}  \sum_j^{N_b}  \hat{\pi}(x, z_j)  \mathds{1}\{\phi_\theta(z_j)=\ell\})$ \;
    }
     $ \hat{b}_{\phi_\theta}^+(x) = \min_{l, m} \hat{b}^+_{{\phi_\theta}; l, m}(x) ,  \quad \hat{b}_{\phi_\theta}^-(x) = \max_{l, m} \hat{b}^-_{{\phi_\theta};l, m}(x)$ for $l, m \in \{1, ..., K\}$ \; 

     $\mathcal{L}(\theta) \gets \mathcal{L}_\mathup{b}(\theta) + \lambda \mathcal{L}_\mathup{reg}(\theta) + \gamma \mathcal{L}_\mathup{aux}(\theta)$ as per Sec.~\ref{sec:neural_method}\;
     $\theta \gets \theta - \delta\nabla_{\theta} \mathcal{L(\theta)}$ \;
     }

\tcp{Final bounds}
\textbf{return} $\hat{b}_{\phi_\theta}^-(x), \hat{b}_{\phi_\theta}^+(x)$ \;
\end{algorithm}
\circled{2}~\textbf{Representation learning:} In the second stage, we train a neural network to learn discrete representations of the instruments with the objective of obtaining tight bounds but with constraints on the estimation variance. To learn the function $\phi(z)$, we use a neural network $\phi_\theta$ with trainable parameters $\theta$. Then, on top of the final layer of the encoder, we leverage the Gumbel-softmax trick \citep{jang2016categorical}, which allows us to learn $k$ \emph{discrete} representations of the latent space of the instruments, where $k$ can be flexibly chosen as a hyperparameter. 

\textbf{Custom loss function:} We further transform our objective into a loss function to train the network $\phi_\theta$. For that, we design a compositional loss consisting of three terms: 

\circledgray{1}~A \emph{bound-width minimization loss} that aims at our objective in Eq.~\eqref{eq:objective_finite_sample}, defined via
\begin{equation}
\mathcal{L}_\mathup{b}(\theta) = \frac{1}{n}\sum_{i=1}^{n}\hat{b}_{\phi_\theta}^+(x_i) - \hat{b}_{\phi_\theta}^-(x_i)    
\end{equation}
\circledgray{2}~A \emph{regularization loss} to enforce the constraints in Eq.~\eqref{eq:objective_finite_sample}, i.e., enforcing that $\hat{p}_{\ell, \Phi} =\hat{\mathbb{P}}(\phi_\theta(Z)=\ell) > \varepsilon$, $\forall \ell \in {1, \ldots ,k}, \text{ for some } \varepsilon > 0$. For that, we aim to penalize the negative log-likelihood $- \sum_{j=1}^k \log(\mathbb{P}({\phi_\theta}(Z)=j))$, which we can estimate via 
\begin{equation}
\mathcal{L}_\mathup{reg}(\theta)  {=} - \sum_{j=1}^k \log\Big(\frac{1}{n}\sum_{i=1}^{n}  \mathds{1}\{{\phi_\theta}(z_i)=j\}\Big).    
\end{equation}

\circledgray{3}~An \emph{auxiliary guidance loss} $\mathcal{L}_\mathup{aux}(\theta)$, which enforces more heterogeneity between $\mathbb{P}(Z \mid \phi_\theta(Z)=l)$ and $\mathbb{P}(Z \mid \phi_\theta(Z)=m)$, for all $l, m$. We implement this by adding an additional linear classification head on top of the last hidden layer of $\phi_\theta$ before the discretization and then calculating the cross entropy loss between the predictions of the head and the discrete representation $\phi_\theta(z)$. In principle, the term $\mathcal{L}_\mathup{aux}(\theta)$ is not needed for our objective, but we empirically observed that it helps to stabilize the training further by avoiding to `get stuck' in non-informative local minima. Hence, we yield our final training loss 
\begin{equation}\label{eq:loss_all}
   \mathcal{L}(\theta) = \mathcal{L}_\mathup{b}(\theta) + \lambda \mathcal{L}_\mathup{reg}(\theta) + \gamma \mathcal{L}_\mathup{aux}(\theta),
\end{equation}
with hyperparameters $\lambda$ and $\gamma$. Here, $\lambda$ controls the trade-off between bound tightness and estimation variance, and can thus be tailored depending on the application. The hyperparameter $\gamma$ can be simply tuned as usual.

A benefit of our custom loss is that it is particularly efficient and robust compared to other learning procedures (such as alternating learning procedure or adversarial training). In the second stage, we solely update the parameters $\theta$ of the discretization network $\phi_\theta$ to minimize $\mathcal{L_\theta}$. In contrast, the networks of the first-stage nuisance estimators have frozen weights. In the second stage, networks of the first-stage nuisance estimators are only evaluated but are \emph{not} updated. 

\section{Experiments}\label{sec:experiments}
\textbf{Baselines:} Existing methods (see Sec.~\ref{sec:related_work}) focus either on (a)~point identification with strong assumptions, (b)~partial identification with continuous treatment variables, or (c)~discrete instruments. We instead focus on a setting with complex instruments and binary treatments. Hence, existing methods are \textbf{not} tailored to our setting, because of which a fair comparison is precluded. Instead, we thus demonstrate the validity and tightness of our bounds. Further, for comparison, we propose an additional \textsc{Na{\"i}ve} baseline, which first learns a discretization of the instruments (via $k$-means clustering) and then learns the nuisance functions wrt. to the discretized instruments to apply the existing bounds for discrete instruments from Lemma~\ref{lemma:bounds_discrete} on top. 

\begin{wrapfigure}[15]{r}{0.7\textwidth}
\vspace{-0.2cm}
\centering
\begin{subfigure}
  \centering
  \includegraphics[height=.33\linewidth]{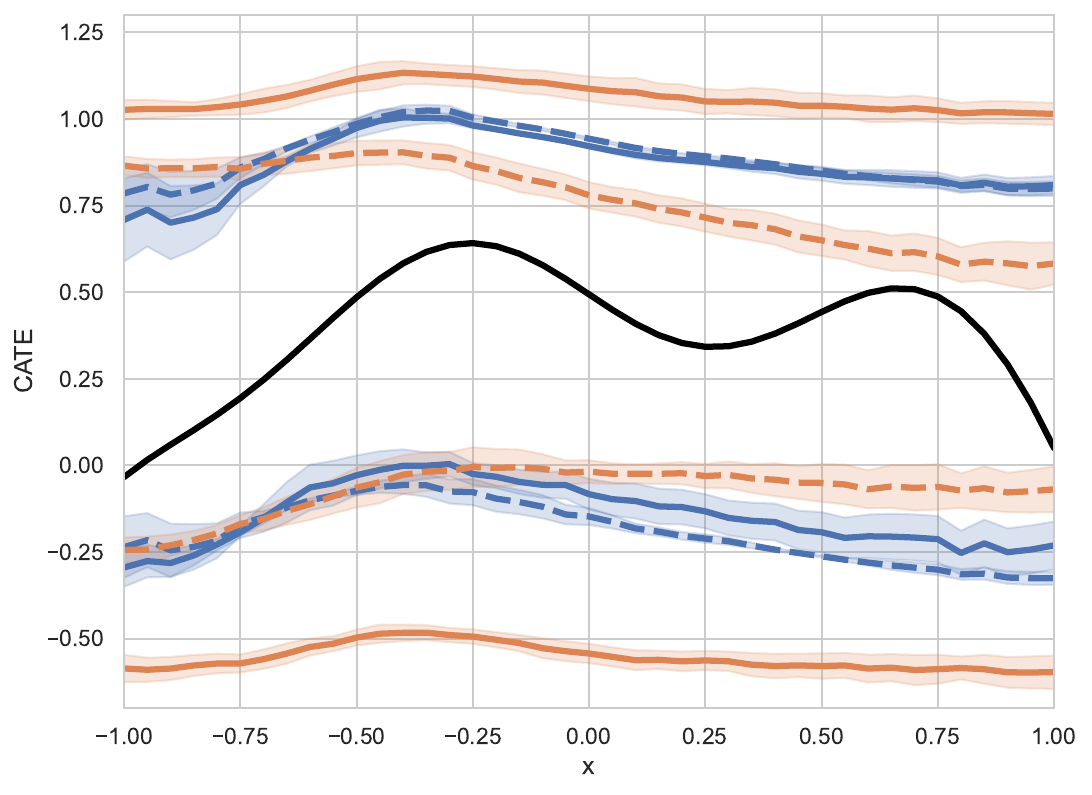}
\end{subfigure}%
\begin{subfigure}
  \centering
\includegraphics[height=.33\linewidth]{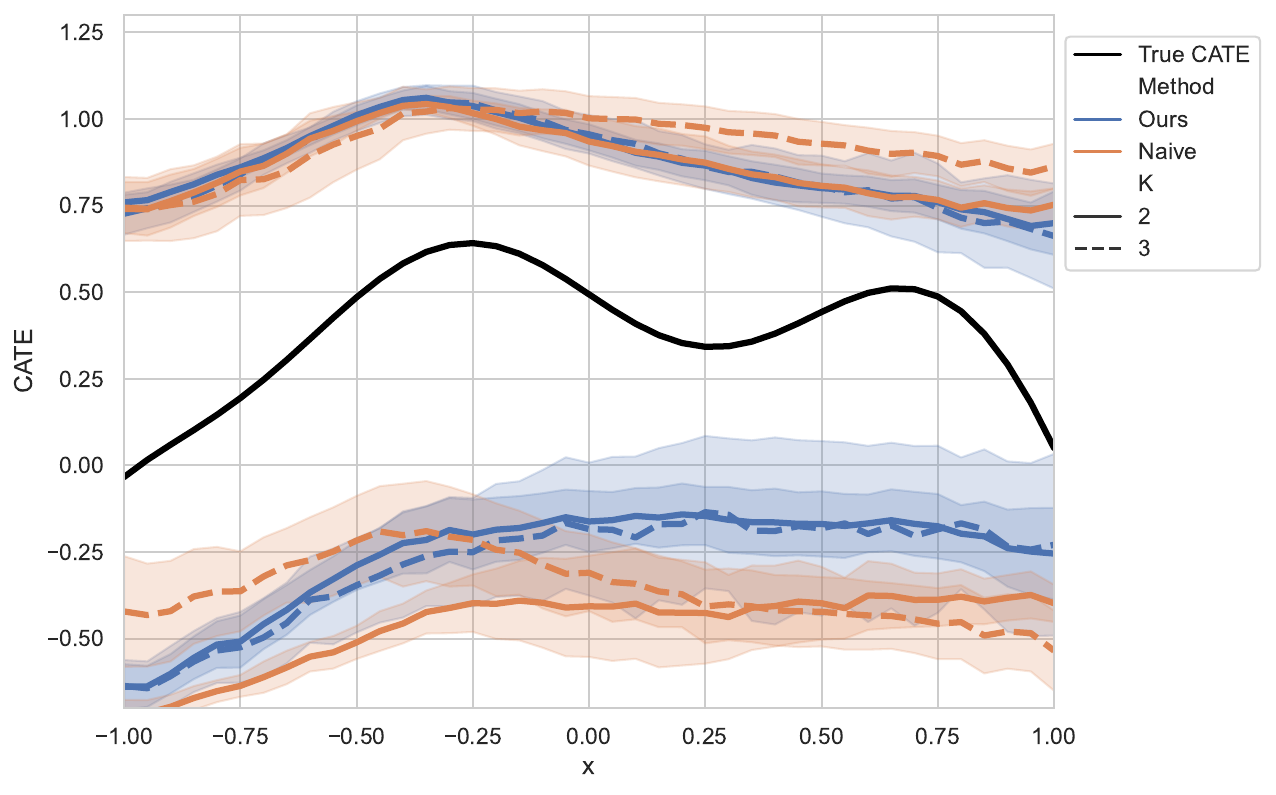}
\end{subfigure}%
\vspace{-0.2cm}
\caption{\textbf{Datasets 1 and 2: Estimated bounds on the CATE.} Shown: mean $\pm$ sd over 5 runs for different number of discretizations $k$. \emph{Left}: Dataset 1 with a  simple $\pi(x, z)$. \emph{Right}: Dataset 2 with a complex $\pi(x, z)$.}
\label{fig:results_synth}
\end{wrapfigure}
\textbf{Data:} We perform experiments mimicking Mendelian Randomization but where we simulate the data to have access to the ground-truth CATE for performance evaluations, so that we can check for coverage and validity of the bounds. We consider three different realistic settings. For Datasets 1 and 2, we consider a one-dimensional continuous instrument representing a polygenic risk score \citep{pierce2018mendelian}. Further, in Dataset~1, we model the true $\pi(x, z)$ as a rather simple function to check if our method is already competitive in such settings. In Dataset~2, we model $\pi(x, z)$ as a complex function to evaluate the performance in more challenging settings.  We use the same CATE for Dataset~1 and Dataset~2 to allow for comparisons between both. In Dataset 3, we model high-dimensional instruments with single nucleotide polymorphisms \cite[SNPs, i.e., genetic variants;][]{burgess2020robust} to test our method in an additional realistic and even more complex setting. In all datasets, we model the CATE to be heterogeneously conditioned on $X$ to check whether the bounds adapt to different subpopulations. Details are in Appendix~\ref{appendix:data}.

\textbf{Performance metrics:}  We report the following metrics to assess the validity and robustness of the estimated bounds: (i)~The \emph{coverage}, i.e., how often the true CATE lies within the estimated bounds. (ii)~The average \emph{width} of bounds, where lower values indicate more informative bounds. (iii)~The \emph{mean squared difference} MSD($k$) of the predicted bounds over different values of $k$, indicating the robustness wrt. to the selection of the hyperparameter. Further, for Dataset 3, we model $\pi(x, z)$ to be dependent on some latent discrete representation of the observed $Z$, such that we can approximate oracle bounds. Thus, we can evaluate the (iv) MSE and (v) the coverage wrt. to the oracle bounds.

\begin{table}[h]
\centering
\scriptsize
\begin{tabular}{llcccc}
\toprule
Dataset & Method & $k$ & Coverage[$\uparrow$] & Width[$\downarrow$] & MSD($k$)[$\downarrow$] \\
\midrule
\multirow{4}{*}{Dataset 1}
    & Na{\"i}ve  & 2 & $\textbf{1.00} \pm 0.00$ & $1.62 \pm 0.06$ & \multirow{2}{*}{$0.28 \pm 0.06$}  \\
    & Na{\"i}ve  & 3 & $\textbf{1.00} \pm 0.00$ & $\textbf{0.83} \pm 0.16$ &  \\
    \cline{2-6}
    & Ours       & 2 & $\textbf{1.00} \pm 0.00$ & $\textbf{1.01} \pm 0.05$ & \multirow{2}{*}{$\textbf{0.03} \pm 0.03$}  \\
    & Ours       & 3 & $\textbf{1.00} \pm 0.00$ & $1.09 \pm 0.04$ & \\
\midrule
\midrule
\multirow{4}{*}{Dataset 2}
    & Na{\"i}ve  & 2 & $\textbf{1.00} \pm 0.00$ & $1.34 \pm 0.19$ & \multirow{2}{*}{$0.09 \pm 0.06$} \\
    & Na{\"i}ve  & 3 & $\textbf{1.00} \pm 0.00$ & $1.28 \pm 0.20$ & \\
    \cline{2-6}
    & Ours       & 2 & $\textbf{1.00} \pm 0.00$ & $\textbf{1.13} \pm 0.19$ & \multirow{2}{*}{$\textbf{0.06} \pm 0.06$} \\
    & Ours       & 3 & $\textbf{1.00} \pm 0.00$ & $\textbf{1.15} \pm 0.31$ &  \\
\bottomrule
\end{tabular}

\caption{\textbf{Datasets 1 and 2: Comparison of the \textsc{Na{\"i}ve} baseline and our method.} Reported are the coverage, width, and mean squared distance MSD($k$) between the estimated bounds over 5 runs. Best results in bold. }
\label{tab:results_synth}
\end{table}
\textbf{Implementation details:} For our method, we use multi-layer-perceptrons (MLPs) for the first-stage nuisance estimation and an MLP with Gumbel-softmax \citep{jang2016categorical} discretization on the last layer for learning $\phi_\theta$. For the \textsc{Na{\"i}ve} baseline, we use $k$-means clustering in the first step to learn discretized instruments and then use MLPs with identical architecture for the nuisance estimation to ensure a fair comparison. We provide further details in Appendix~\ref{sec:app_implementation}.
     
\begin{table}[h]  
\centering
\scriptsize
    \begin{tabular}{ll|p{1.5cm}p{1.5cm}|p{1.5cm}p{1.5cm}|p{1.5cm}}
        \toprule
        Method & $k$ &  Coverage[$\uparrow$] & Width[$\downarrow$] & Coverage (oracle bounds)[$\uparrow$] & MSE (oracle bounds)[$\downarrow$] &  MSD($k$)[$\downarrow$]\\
        \midrule
        Naive & 2 & $1.00 \pm 0.00$ & $1.96 \pm 0.05$ & $1.00 \pm 0.00$ & $0.15 \pm 0.02$ & \multirow{4}{*}{$0.10 \pm 0.10$}\\ 
              & 4 & $1.00 \pm 0.00$ & $1.91 \pm 0.03$ & $1.00 \pm 0.00$ & $0.13 \pm 0.02$ & \\
              & 6 & $1.00 \pm 0.00$ & $1.74 \pm 0.26$ & $0.75 \pm 0.50$ & $0.09 \pm 0.05$ & \\
              & 8 & $1.00 \pm 0.00$ & $1.89 \pm 0.09$ & $1.00 \pm 0.00$ & $0.12 \pm 0.04$ & \\
        \midrule
        Ours  & 2 & $1.00 \pm 0.00$ & $1.87 \pm 0.05$ & $1.00 \pm 0.00$ & $0.12 \pm 0.02$ & \multirow{4}{*}{$\textbf{0.03} \pm 0.02$} \\
              & 4 & $1.00 \pm 0.00$ & $1.87 \pm 0.08$ & $1.00 \pm 0.00$ & $0.12 \pm 0.03$ & \\
              & 6 & $1.00 \pm 0.00$ & $1.85 \pm 0.06$ & $1.00 \pm 0.00$ & $0.11 \pm 0.02$ & \\
              & 8 & $1.00 \pm 0.00$ & $1.83 \pm 0.07$ & $0.99 \pm 0.01$ & $0.11 \pm 0.03$ & \\
        \bottomrule
    \end{tabular}
    \vspace{-0.2cm}
\caption{\textbf{Dataset 3 (high-dimensional): Comparison of the \textsc{Na{\"i}ve} baseline and our method.} Reported are the coverage, width, the coverage and MSE wrt. to the oracle bounds, and the mean squared distance MSD($k$) between the estimated bounds over $5$ runs.}\label{tab:DRest}
\end{table}
    
\begin{wrapfigure}[15]{r}{0.5\textwidth}
  \vspace{-0.2cm}
        \centering
    \includegraphics[width=.9\linewidth]{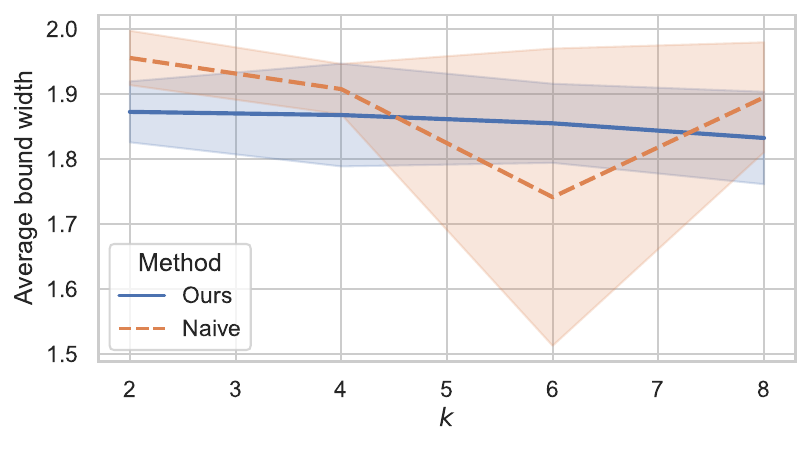}
                \vspace{-0.35cm}
        \caption{\textbf{Dataset 3 (high-dimensional): Average bound width.} Sensitivity analysis wrt. to the number of instruments $k$ where we show the average bound width and estimation variance over 5 runs.}
        \label{fig:bound_width3}
\vspace{-0.1cm}
\end{wrapfigure}

\textbf{Results:} We present the results of our experiments for Datasets 1 and 2 in Table~\ref{tab:results_synth} and Fig.~\ref{fig:results_synth}, and for Dataset 3 in Table~\ref{tab:DRest} and Fig.~\ref{fig:bound_width3}. Here, we compare our method against the \textsc{Na{\"i}ve} baseline and vary the number of clusters $k$, with higher $k$ for the more complex Dataset 3. 

Overall, we observe the following patterns: \textbf{(i)}~Both methods (i.e., ours and the \textsc{Na{\"i}ve} baseline) almost always reach a perfect coverage of 100\%, which shows the validity of the bounds. \textbf{(ii)}~As expected, our method is able to learn \emph{tighter bounds} on the more complex Datasets~2 and 3 over the different runs and for different $k$. This demonstrates that our method aimed at learning representations with the objective of tightening bounds can clearly improve over a discretization that uses solely information of $Z$ in the first step (\textsc{Na{\"i}ve}). \textbf{(iii)}~Our method is robust, especially also over different values of $k$. This is demonstrated by a low MSD($k$), in particular for Datasets~1 and 3. Yet, this is unlike the \textsc{Na{\"i}ve} baseline, which performs \emph{not} robust wrt. $k$. The latter is also clearly visible by the large differences in the learned bounds in Fig.~\ref{fig:results_synth} on the left. Further, Fig.~\ref{fig:bound_width3} indicates unstable behavior of the \textsc{Na{\"i}ve} baseline with higher variation, even resulting in learning falsely overconfident bounds for $k=6$. In contrast, our method yields bounds that are robust for given $k$ as well as over varying values of $k$, which is naturally encouraged by our objective of flexibly learning representations while ensuring robustness in bound estimation.

\textbf{\underline{Takeaways:}} Our method can successfully learn bounds that have a high coverage and a low width. Further, our method outperforms the \textsc{Na{\"i}ve} baseline clearly while ensuring robustness. Here, our results show that the source of the performance gain is the way how we learn the representation $\phi$ and that the performance gain from our method becomes larger for more complex datasets.

\textbf{\underline{Limitations:}} Our method for partial identification allows us to relax multiple assumptions that are inherent to methods for point identification. Nevertheless, we still rely on the standard assumptions of IV settings. However, such assumptions often hold by design or can be ensured by expert knowledge such as in Mendelian randomization. We provide an extended discussion in Appendix~\ref{app:RW_assumptions}.

\textbf{\underline{Conclusion:}} We propose a novel method for learning tight bounds on treatment effects by making use of complex instruments (e.g., instruments that are continuous, potentially high-dimensional, and that have non-trivial relationships with the treatment intake or exposure). 

\newpage

\bibliography{literature}
\bibliographystyle{iclr2025_conference}

\newpage

\appendix

\section{Proofs}\label{app:proofs}

\subsection{Proof of Theorem~\ref{thrm:bounds_phi}}

We begin by stating a result from the literature that obtains valid bounds for discrete instruments.
\begin{lemma}[\citep{swanson2018partial, schweisthal2024meta-learners}]\label{lemma:bounds_discrete}
Under Assumptions~\ref{ass:consistency} and \ref{ass:exclsion}, the CATE is bounded via
\begin{equation}
b^-(x) \leq \tau(x) \leq b^+(x),
\end{equation}
with
\begin{equation}\label{eq:bounds_discrete}
    b^+(x) = \min_{l, m} b^+_{l, m}(x) \quad \text{and} \quad b^-(x) = \max_{l, m} b^-_{l, m}(x) 
\end{equation}
where
\begin{equation}
    b^+_{l, m}(x) =\pi(x, l) \mu^{1}(x, l) + (1 - \pi(x, l)) s_2   - (1-\pi(x, m)) \mu^{0}(x, m) -  \pi(x, m) s_1 ,
\end{equation}
\vspace{-0.4cm}
\begin{equation}
    b^-_{l, m}(x) =\pi(x, l) \mu^{1}(x, l) + (1 - \pi(x,  l)) s_1   - (1-\pi(x, m)) \mu^{0}(x, m) -  \pi(x, m) s_2.
\end{equation}
\end{lemma}

\begin{proof}[Proof of Theorem~\ref{thrm:bounds_phi}]
First, note that, for a given representation $\phi$, the representation $\phi(Z)$ is still a valid (discrete) instrument that satisfies Assumptions~\ref{ass:consistency} and \ref{ass:exclsion}. Hence, we can apply Lemma~\ref{lemma:bounds_discrete} using $\phi(Z)$ as an instrument and immediately obtain the bounds from Theorem~\ref{thrm:bounds_phi}, but with \emph{representation-induced nuisance functions}
$\mu_{\phi}^a(x, \ell) = \mathbb{E}[Y|X=x, A=a, \phi(Z)=\ell]$ and $\pi_{\phi}(x, \ell) = \mathbb{P}(A=1 | X=x, \phi(Z)=\ell)$ for $\ell \in \{0, \dots, k\}$. 

We can write the representation-induced response function as
\begin{equation}\label{eq:mu_transform}
    \begin{aligned}
       & \mathbb{E}[Y|X=x, A=a, \phi(Z)=\ell] \stackrel{Z \indep X}{=} \int_Z \mathbb{E}[Y|X=x, A=a, Z=z] \mathbb{P}(Z=z|A=a, \phi(Z)=\ell) \diff z  \\
       &= \int_Z \mathbb{E}[Y|X=x, A=a, Z=z] \frac{\mathbb{P}(\phi(Z)=\ell|A=a, Z=z) \mathbb{P}(A=a| Z=z) \mathbb{P}(Z=z)}{\mathbb{P}(A=a| \phi(Z)=\ell) \mathbb{P}(\phi(Z)=\ell)}\diff z \\
       &= \frac{1}{\mathbb{P}(A=a| \phi(Z)=\ell) \mathbb{P}(\phi(Z)=\ell)} \\ 
       & \qquad \int_Z \mathbb{E}[Y|X=x, A=a, Z=z]  \mathbb{P}(\phi(Z)=\ell|A=a, Z=z) \mathbb{P}(A=a| Z=z) \mathbb{P}(Z=z)\diff z \\
       &= \frac{1}{\mathbb{P}(A=a| \phi(Z)=\ell) \mathbb{P}(\phi(Z)=\ell)} \\ 
       & \qquad\qquad \int_Z \mathbb{E}[Y|X=x, A=a, Z=z]  \mathbb{P}(\phi(Z)=\ell| Z=z) \mathbb{P}(A=a| Z=z) \mathbb{P}(Z=z) \diff z 
    \end{aligned}
\end{equation}
and the representation-induced propensity score as
\begin{equation}\label{eq:pi_transform}
    \begin{aligned}
      & \mathbb{P}(A=1 | X=x, \phi(Z) = \ell) \stackrel{Z \indep X}{=} \int_Z \mathbb{P}(A=1 | X=x, Z = z) \mathbb{P}(Z=z | \phi(Z) = \ell) \diff z  \\
       =& \int_Z \mathbb{P}(A=1 | X=x, Z = z) \mathbb{P}(\phi(Z)=\ell | Z=z) \frac{\mathbb{P}(Z=z)}{\mathbb{P}(\phi(Z)=\ell)} \diff z  \\
       =& \frac{1}{\mathbb{P}(\phi(Z)=\ell)} \int_Z \mathbb{P}(A=1 | X=x, Z = z) \mathbb{P}(\phi(Z)=\ell | Z=z) \mathbb{P}(Z=z)\diff z,
    \end{aligned}
\end{equation}
which completes the proof.
\end{proof}

\subsection{Proof of Lemma~\ref{lem:error_decomposition}}
\begin{proof}
The result follows from 
\begin{align}
    \mathbb{E}_n \left[\left(b_\ast^+(x) - \hat{b}_{\phi^\ast}^+(x)\right)^2\right] & = \mathbb{E}_n \left[\left(b_\ast^+(x) - {b}_{\phi^\ast}^+(x) + {b}_{\phi^\ast}^+(x) - \hat{b}_{\phi^\ast}^+(x)\right)^2\right] \\
    & \leq 2 \left( \left(b_\ast^+(x) - {b}_{\phi^\ast}^+(x)\right)^2 + \mathbb{E}_n \left[\left({b}_{\phi^\ast}^+(x) - \hat{b}_{\phi^\ast}^+(x)\right)^2\right] \right)\\
    & \overset{(\ast)}{(=)} 2 \left( \left(b_\ast^+(x) - {b}_{\phi^\ast}^+(x)\right)^2 + \mathbb{E}_n\left[b_{\phi^\ast}^+(x) - \hat{b}_{\phi^\ast}^+(x)\right]^2 + \mathup{Var}_n(\hat{b}_{\phi^\ast}^+(x)) \right),
\end{align}
where we used the bias-variance decomposition for the MSE for $(\ast)$.
\end{proof}

\subsection{Proof of Theorem~\ref{thrm:variances}}

\begin{proof}
We derive the asymptotic distributions of the estimators $\hat{\mu}_{\phi}^{a}(x, \ell)$ from Eq.~\eqref{eq:estimates_nuisance_mu} and $\hat{\pi}_{\phi}(x, \ell)$ from Eq.~\eqref{eq:estimates_nuisance_phi}. We proceed by analyzing the numerator and denominator of each estimator. First, we show that both are asymptotically normal and then we apply the delta method to obtain the asymptotic distribution of the ratios.

\textbf{Distribution of $\hat{\mu}_{\phi}^{a}(x, \ell)$:} Recall from Equation~\eqref{eq:estimates_nuisance_mu} that we can write $\hat{\mu}_{\phi}^{a}(x, \ell)$ as

\begin{equation}
\hat{\mu}_{\phi}^{a}(x, \ell) = \frac{S_n}{N_n},
\end{equation}
where

\begin{align}
S_n &= \frac{1}{n}\sum_{j=1}^{n} W_j, \quad \text{with} \quad W_j = \hat{\mu}^{a}(x, z_j) \mathds{1}\{\phi(z_j) = \ell\} [a \hat{\eta}(z_j) + (1 - a)(1 - \hat{\eta}(z_j))], \\
N_n &= \frac{1}{n}\sum_{j=1}^{n} D_j, \quad \text{with} \quad D_j = \mathds{1}\{\phi(z_j) = \ell, a_j = a\}.
\end{align}

We define the moments
\begin{align}
\mu_W &= \mathbb{E}[W] = p_\ell \theta_\ell \\
\sigma_W^2 &= \operatorname{Var}(W) = p_\ell (\gamma_\ell - p_\ell \theta^2_\ell) \\
\mu_D &= \mathbb{E}[D] = p_\ell q_\ell\\
\sigma_D^2 &= \operatorname{Var}(D) = p_\ell q_\ell (1 - p_\ell q_\ell) \\
\mathrm{c}_{WD} &= \operatorname{Cov}(W, D) = p_\ell q_\ell \theta_\ell (1 - p_\ell),
\end{align}
where $p_\ell = \mathbb{P}(\Phi(Z) = \ell)$, $q_\ell = \mathbb{P}(A = a \mid \Phi(Z) = \ell)$, $\theta_\ell = \E[g(Z) \mid \Phi(Z) = \ell]$, and $\gamma_\ell = \E[g(Z)^2 \mid \Phi(Z) = \ell]$, with $g(Z) = \hat{\mu}^{a}(x, Z) (a \hat{\eta}(Z) + (1 - a)(1 - \hat{\eta}(Z))$.

By the central limit theorem, we know that
\begin{equation}
\sqrt{n} \begin{pmatrix} S_n \\ N_n \end{pmatrix} \xrightarrow{d} \mathcal{N}_2\left(\mu = \begin{pmatrix} \mu_W \\ \mu_D \end{pmatrix}, \Sigma = \begin{pmatrix} \sigma_W^2 & \mathrm{c}_{WD}\\ \mathrm{c}_{WD} & \sigma_D^2 \end{pmatrix}\right).
\end{equation}

Let $f(s, n) = \frac{s}{n}$. We are interested in the asymptotic distribution of the ratio $\hat{\mu}_{\phi}^{a}(x, \ell) = f(S_n,{N_n})$. The delta method states that

\begin{equation}
\sqrt{n} f(S_n, N_n) \xrightarrow{d} \mathcal{N}_2\left(f(\mu_W, \mu_D), \nabla f^\top (\mu_W, \mu_D) \Sigma \nabla f (\mu_W, \mu_D) \right)    
\end{equation}

Using that the gradient is $\nabla f^\top (\mu_W, \mu_D) = \left(\dfrac{1}{\mu_D}, -\dfrac{\mu_W}{\mu_D^2}\right)$, we can obtain the asymptotic variance via

\begin{align}
\nabla f^\top (\mu_W, \mu_D) \Sigma \nabla f (\mu_W, \mu_D) &= \dfrac{\sigma_W^2}{\mu_D^2} - 2 \dfrac{\mu_W \mathrm{c}_{WD}}{\mu_D^3} + \dfrac{\mu_W^2 \sigma_D^2}{\mu_D^4} \\
&= \frac{1}{p_\ell}\left(\frac{(\gamma_\ell - \theta_\ell^2)}{q_\ell^2} + \frac{\theta_\ell^2(1 - p_\ell q_\ell)}{q_\ell^3} \right) \\
& = \frac{1}{p_\ell}\left(\frac{\operatorname{Var}(g(Z) \mid \Phi(Z) = \ell)}{q_\ell^2} + \frac{\theta_\ell^2(1 - p_\ell q_\ell)}{q_\ell^3} \right).
\end{align}

\textbf{Distribution of $\hat{\pi}_{\phi}(x, \ell)$:} Recall from Equation~\eqref{eq:estimates_nuisance_phi} that we can write $\hat{\pi}_{\phi}(x, \ell)$ as

\begin{equation}
\hat{\pi}_{\phi}(x, \ell) = \frac{S_n}{N_n},
\end{equation}
where

\begin{align}
S_n &= \frac{1}{n}\sum_{j=1}^{n} W_j, \quad \text{with} \quad W_j = \hat{\pi}(x, z_j)  \mathds{1}\{\phi(z_j)=l\}, \\
N_n &= \frac{1}{n}\sum_{j=1}^{n} D_j, \quad \text{with} \quad D_j = \mathds{1}\{\phi(z_j)=l\}.
\end{align}

We define the moments
\begin{align}
\mu_W &= \mathbb{E}[W] = p_\ell \theta_\ell \\
\sigma_W^2 &= \operatorname{Var}(W) = p_\ell (\gamma_\ell - p_\ell \theta^2_\ell) \\
\mu_D &= \mathbb{E}[D] = p_\ell\\
\sigma_D^2 &= \operatorname{Var}(D) = p_\ell (1 - p_\ell) \\
\mathrm{c}_{WD} &= \operatorname{Cov}(W, D) = p_\ell \theta_\ell (1 - p_\ell),
\end{align}
where $p_\ell = \mathbb{P}(\Phi(Z) = \ell)$, $\theta_\ell = \E[h(Z) \mid \Phi(Z) = \ell]$, and $\gamma_\ell = \E[h(Z)^2 \mid \Phi(Z) = \ell]$, with $h(Z) = \hat{\pi}(x, Z)$.

By the central limit theorem, we know that
\begin{equation}
\sqrt{n} \begin{pmatrix} S_n \\ N_n \end{pmatrix} \xrightarrow{d} \mathcal{N}_2\left(\mu = \begin{pmatrix} \mu_W \\ \mu_D \end{pmatrix}, \Sigma = \begin{pmatrix} \sigma_W^2 & \mathrm{c}_{WD}\\ \mathrm{c}_{WD} & \sigma_D^2 \end{pmatrix}\right).
\end{equation}

We can then calculate the asymptotic variance using the delta method as above and obtain
\begin{align}
\nabla f^\top (\mu_W, \mu_D) \Sigma \nabla f (\mu_W, \mu_D) &= \dfrac{\sigma_W^2}{\mu_D^2} - 2 \dfrac{\mu_W \mathrm{c}_{WD}}{\mu_D^3} + \dfrac{\mu_W^2 \sigma_D^2}{\mu_D^4} \\
&= \frac{1}{p_\ell}(\gamma_\ell - \theta_\ell^2) \\
& = \frac{1}{p_\ell}\operatorname{Var}(h(Z) \mid \Phi(Z) = \ell).
\end{align}

\end{proof}

\clearpage

\section{Real-world relevance and validity of assumptions}
\label{app:RW_assumptions}

In this section, we elaborate on the real-world relevance of our considered setting and show that our assumptions often hold and are even weaker than the ones of existing approaches. For that, we draw upon two real-world settings.

\subsection{Mendelian randomization}
Mendelian randomization (MR; the main motivational example from our paper) is a widely used method from biostatistics to estimate the causal effect of some treatment or exposure (such as alcohol consumption) on some outcome (such as cardiovascular diseases). We refer to \cite{pierce2018mendelian} for an introduction to MR, which also shows that MR is widely used in medicine. For that, genetic variants (such as different single nucleotide polymorphisms, SNPs) are used as instruments where it is known that they only influence the exposure but not directly the outcome. Our method for partial identification with complex instruments is perfectly suited for this common real-world application. Depending on the use case, either a predefined genetic risk score \citep{burgess2020robust} as a continuous variable, or up to hundreds of SNPs are used simultaneously as IVs to strengthen the power of the analysis, resulting in high-dimensional instruments \citep{pierce2018mendelian}.

\textbf{Validity of assumptions:} The IV assumptions used in our paper such as the exclusion and independence assumptions can be ensured by expert knowledge (e.g.,  given some observed confounder age ($X)$, genetic variations (Z) do not affect age) or, in some cases,  they can be even directly tested for \citep{glymour2012credible}. In contrast, as explained in Sec.~\ref{sec:related_work}, existing methods for MR rely on additional hard assumptions on top such as the knowledge about the parametric form of the underlying data-generating process. Especially with such high-dimensional IVs, misspecification of these models may result in significantly biased effect estimates. In contrast, our method does not rely on any parametric assumption and also no additional assumptions compared to previous methods, thus enabling more reliable causal inferences in the real-world application of MR by using \emph{strictly weaker} assumptions than existing work.

\subsection{Indirect experiments}
With indirect experiments (IEs), we show that, in principle, our method is not constrained to medical applications but is also highly useful in various other domains. IEs are widely applied in various areas such as social sciences or public health to estimate causal effects in settings with non-adherence, i.e., where people cannot be forced to take treatments but rather be encouraged by some nudge \citep{pearl1995causal}. For instance, researchers might be interested in estimating the effect of some treatment such as participating in a healthcare program ($T$) on some health outcome $Y$ by randomly assigning nudges $Z$ (IVs) in the form of different text messages on social media promoting participation. Here, common nudges (IVs) are in the form of, for instance, text or even image data and thus high-dimensional, showing the necessity of a method capable of handling complex IVs such as ours.

\textbf{Validity of assumptions:} As a major benefit of IEs, the IV assumptions are \emph{ensured per design} as the IVs are randomly assigned, and, thus they always hold. Hence, our method provides a promising tool for evaluating the effects of IEs.

\newpage

\section{Implementation and training details}\label{sec:app_implementation}

\textbf{Model architecture:} For all our models, we use MLPs with ReLU activation function. For $\hat{\mu}_{\phi}^{a}$, we use 2 layers to encode $X$ and 3 layers to encode $Z$. Then, we concatenate the outputs and add 2 additional shared layers. Finally, we calculate the outputs by a separate treatment head for $A=0$ and $A=1$ to ensure the expressiveness of $A$ for predicting $Y$. For $\hat{\pi}$, we use the same architecture. For $\hat{\eta}$, we use 3 layers. For $\phi_\theta$, we also use 3 layers and apply discretization on top of the $K$ outputs \citep{jang2016categorical}. For the nuisance parameters of the $k$-means baseline, we use the same models as for $\hat{\mu}_{\phi}^{a}$ and $\hat{\pi}$ for a fair comparison. We use a neuron size of 10 for all hidden layers.

\textbf{Training details:} For training our nuisance functions, we use an MSE loss for the functions learning the continuous outcome $Y$ and a cross-entropy loss for functions learning the binary treatment $A$. For all models, we use the Adam optimizer with a learning rate of $0.03$. We train our models for a maximum of 100 epochs and apply early stopping. For our method, we fixed $\lambda=1$ and performed random search to tune for $[0, 1]$ for $\gamma$. We use PyTorch Lightning for implementation. Each training run of the experiments could be performed on a CPU with 8 cores in under 15 minutes.

\newpage

\section{Data description}\label{appendix:data}
\textbf{Dataset 1:} We simulate an observed confounder $X \sim \mathrm{Uniform}[-1, 1]$ and an unobserved confounder $U \sim \mathrm{Uniform}[-1, 1]$. 

The instrument $Z$ is defined as
\begin{equation}
Z \sim \text{Mixture}\left(\frac{1}{2} \mathrm{Uniform}[-1, 1] + \frac{1}{4} \mathrm{Beta}(2, 2) + \frac{1}{4} (-\mathrm{Beta}(2, 2))\right).
\end{equation}

We define $\rho$ as
\begin{equation}
\rho = \frac{1}{1 + \exp\left(-\left( \left(2 |Z| - \max(Z)\right) + X + 0.5 \cdot U\right)\right)}.
\end{equation}

Then, the propensity score is given by
\begin{equation}
\pi = (\rho - 0.5) \cdot 0.9 + 0.5.
\end{equation}

We then sample our treatment assignments from the propensity scores as
\begin{equation}
A \sim \textrm{Bernoulli}({\pi}).
\end{equation}

The conditional average treatment effect (CATE) is defined as
\begin{equation}
\tau(X) = -\frac{(2.5X)^4 + 12 \sin(6X) + 0.5 \cos(X)}{80} + 0.5.
\end{equation}

The outcome $Y$ is then generated by
\begin{equation}
Y = \left(X + 0.5 U + 0.1 \cdot \text{Laplace}(0, 1)\right) \cdot 0.25 + \tau(X) \cdot A.
\end{equation}

\textbf{Dataset 2:} We keep the other properties but change the propensity score to be more complex, which results in harder-to-learn optimal representations of $Z$ for tightening the bounds.
The propensity score is given by
\begin{equation}
\pi = \sin(2.5Z + X + U) \cdot 0.48 + 0.48 + \frac{0.04}{1 + \exp(-3 |Z|)}.
\end{equation}

\textbf{Dataset 3:} We simulate $X$ and $U$ as above. Then, we sample a $d$-dimensional $Z \in \{0, 1\}^{d}$ with $d=20$ as
\begin{equation}
    Z \sim \text{Binomial}(d, 0.5).
\end{equation}
Thus, our modeling is here inspired by using multiple SNPs (appearances of genetic variations) as instruments \citep{burgess2020robust}, where we simulate potential variations for 20 genes.

Then, we define 
\begin{equation}\label{eq:p_d3}
\rho = \sum_{j=1}^d [\mathds{1}\{j \leq 5 \} Z_j]
\end{equation}
and the propensity score, inspired by the more complex setting of Dataset 2, as 
\begin{equation}
    \pi = 0.48\;\text{sin}(10 \rho + X + U) + 0.48 + \frac{0.04}{1+\exp(-3|5 \rho |)}.
\end{equation}
Then, we define the CATE as
\begin{equation}
\tau(X) = -\frac{-(1.6X+0.5)^4 + 12 \sin(4X + 1.5) + \cos(X)}{80} + 0.5.
\end{equation}
and the outcome dependent on $\tau$, $X$ and $U$ analogously as for Datasets 1 and 2.

By Eq.~\eqref{eq:p_d3}, we ensure that some of the modeled SNPs are irrelevant for $\pi$ and thus do not affect the treatment or exposure $A$. Thereby, we focus on realistic settings in practice, where the relevance of instruments cannot always be ensured which imposes challenges especially for existing methods for point identification, but not for our approach. Further, we ensure that the latent score $\rho$ can only take 5 discrete levels. This allows us to approximate oracle bounds using the discrete bounds on top of $\rho$ by leveraging Lemma~\ref{lemma:bounds_discrete} such that we can evaluate our method and the baseline in comparison to oracle bounds.

To create the simulated data used in Sec.~\ref{sec:experiments}, we sample $n=2000$ from the data-generating process above. We then split the data into train (40\%), val (20\%), and test (40\%) sets such that the bounds and deviation can be calculated on the same amount of data for training and testing.

\end{document}